\documentclass[twoside]{article}

\usepackage{microtype}
\usepackage{graphicx,graphbox}
\usepackage[ruled,vlined]{algorithm2e}
\usepackage{subfigure,color}
\usepackage{booktabs} 
\usepackage{amssymb,mathtools,marginnote}
\usepackage{amsthm,bbm,amssymb}
\newtheorem{thm}{Theorem}
\newtheorem{lem}[thm]{Lemma}
\newtheorem{cor}[thm]{Corollary}

\newtheorem{defin}[thm]{Definition}

\usepackage[accepted]{aistats2021}
\begin{document}

%

%

\twocolumn[

\aistatstitle{Concentric mixtures of Mallows models for top-$k$ rankings: \\ sampling and identifiability}

\aistatsauthor{ Fabien Collas  \And Ekhine Irurozki }

\aistatsaddress{ BCAM \And  BCAM, Telecom Paris } ]

\begin{abstract}
In this paper, we consider mixtures of two Mallows models for top-$k$ rankings, both with the same location parameter but with different scale parameters, i.e., a mixture of concentric Mallows models.
This situation arises when we have a heterogeneous population of voters formed by two homogeneous populations, one of which is a subpopulation of expert voters while the other includes the non-expert voters. We propose efficient sampling algorithms for Mallows top-$k$ rankings. We show the identifiability of both components, and the learnability of their respective parameters in this setting by, first, bounding the sample complexity for the Borda algorithm with top-$k$ rankings and second, proposing polynomial time algorithm for the separation of the rankings in each component. 
Finally, since the rank aggregation will suffer from a large amount of noise introduced by the non-expert voters, we adapt the Borda algorithm to be able to recover the ground truth consensus ranking which is especially consistent with the expert rankings. 
\end{abstract}

\section{INTRODUCTION}



Ranked data has been subject to study in different communities starting with Social Choice~\cite{Bartholdi1989a}, Bioinformatics~\cite{Vitelli2018}, and recently in Machine Learning~\cite{DBLP:conf/icml/Busa-FeketeHS14} since rankings arise naturally when ordering items in order of preference. Top-$k$ rankings arise in practice when voters, human or software, see all the items but provide a ranking of their most preferred $k$ items. Examples of top-$k$ rankings are the results displayed in a search engine, which contain just the top 10 most relevant, related search results, out of possibly millions of results. 

The Mallows model (MM)~\cite{gMallows,mallows} is one of the preferred distributions to model rank data. It belongs to the location-scale family since it is parametrized by a location parameter (a.k.a. central ranking) $\sigma_0$ and a non-negative scale (a.k.a. dispersion) parameter, $\theta$. The location parameter is the consensus ranking of the distribution. The probability of any other permutation decreases exponentially with its distance to $\sigma_0$, where the distance for rankings is, in general, the Kendall's-$\tau$ distance. Finally, the dispersion parameter controls the variance of this decay. For other distances see~\cite{Irurozki2016b}. 

Mixtures of MM model populations that are divided into different sub-populations, each of which can be modeled with a single MM since each is consistent with noisy realizations of their particular consensus ranking. In this paper, we study the particular context of mixtures where all the location parameters are the same $\sigma_0$. We denote this situation as \textit{concentric mixture}. 


\paragraph{Real world motivation} 
In this work, we consider the following problem. There is a consensus ranking $\sigma_0$ representing a complete ranking of a set of $n$ alternatives, i.e., films that are ordered for the preferences or students that are ranked according to their grades in a particular exam. However, this consensus ranking is unknown and must be estimated from the realization of rankings provided by a collection of $m$ raters or voters. Each voter has ranked his top $k<n$ alternatives.

The population is heterogeneous: a number of them are expert voters whose vote will be close to the consensus and the rest will provide low-quality rankings, which will be noisier than the experts'. The population of voters is modeled as a 2-components concentric mixture of MM: Both components will be centered at the same consensus ranking $\sigma_0$, but their spread parameter will be different. Our goals are to show identifiability in this scenario by (1) obtaining the consensus ranking with high probability and (2) distinguishing the experts from the non-experts rankings with high probability.

\paragraph{Our contributions}
We study here concentric mixtures of Mallows for top-$k$ rankings. Our contributions are the following:
\begin{itemize}
    \item We propose efficient algorithms to compute the probability, sample top-$k$ rankings, and sample linear extensions of top-$k$ rankings under the Mallows model.

    \item We show identifiability of concentric mixtures with the following two results. 
    \begin{itemize}
        \item We analyze the sample complexity that the Borda algorithm needs to return the consensus ranking with high probability for a sample of top-$k$ rankings.
        \item We propose an algorithm that separates the rankings from both components of the mixture with high probability. This is the key to estimate the dispersion parameters of each component. 
    \end{itemize}

    \item We propose an improvement for the Borda algorithm for the estimation of the central ranking in concentric mixtures of top-$k$ rankings. 

\end{itemize}
\paragraph{Related work}

Partially ranked data and extensions of the Kendall's-$\tau$ distance, in particular, have been analyzed extensively. In \cite{Fagin2003} a family of extensions is studied and the authors show that they are equivalent up to global constants. 
Based on this work, constant factor approximation algorithms can be found in \cite{Ailon2010}.  Mallows models for top-$k$ rankings under the distance in~\cite{Fagin2003} are given in \cite{NIPS2018_7691}, where the authors argue that previous sampling algorithms based on the Repeated Insertion Model (RIM) \cite{Regenwetter2004} can not be efficiently adapted to sample top-$k$ rankings and propose a $\mathcal{O}(k^2 4^k + k^2\log n)$ algorithm to sample top-$k$ rankings from a MM. By taking a different approach from theirs, we propose a sampling algorithm of complexity $\mathcal{O}(k\log k)$ for top-$k$ rankings under the MM. 

Theoretical identifiability of the parameters of a mixture of MM was first addressed in \cite{awasthi2014learning} after a large number of papers on practical research~\cite{D'Elia2005917,Lee20122486,Meila2010}. They obtain a polynomial time algorithm for the case of two mixtures by using tensor decomposition. Despite working with arbitrary separation of the centers of the distributions, their algorithm performance drops as the centers of the distributions get close, being able to correctly identify 10\% of the mixtures when both centers are the same.

The problem of learning mixtures of MM was also addressed in \cite{Chierichetti2015}. They propose and analyze algorithms for different mixture setting. They show identifiability when the dispersion parameter is the same and known for all the components and argue that the learnability of the problem can strongly depend on the separation of the consensus rankings of the mixtures components. In the present work, we focus on this alleged difficult setting of concentric mixtures of unknown dispersion and show that even concentric mixtures can be identified polynomially provided that the variances in both distributions are different enough. 


The first polynomial time algorithm for provably learning the parameters of a mixture of Mallows models with any constant number of components can be found in \cite{liu2018efficiently}. They show that any two mixtures of top-$k$ Mallows models whose components are far from
each other and from the uniform distribution in total variation distance, and also far from the uniform, are far from each other as mixtures too, provided that $n > 10k^2$. 
We improve their results showing that for a component close to uniformity and smaller values of $n$ the separation can be done polynomially. 

A growing body of recent papers consider simultaneously partial preferences and mixtures of probabilistic models, i.e., Plackett-Luce \cite{Mollica2017}, 
proposing provably efficient learning algorithms \cite{Liu2019}, 
sampling linear extensions \cite{zhao2019learning}, 
characterizing  identifiability \cite{zhao2020learning}. 
In this work, we extend the scenario to the Mallows model.

Mallows model belongs to the location-scale family of distributions. The most prominent member of this family is the Gaussian and therefore both Mallows and Gaussian are usually considered to be analogous. Nonetheless, Mallows lacks many interesting properties of the Gaussian. In this paper, we show that on the other hand, Mallows has interesting properties that are not present on the Gaussian such as identifiability of concentric mixtures.

This paper is organized as follows. Section~\ref{sec:perliminaries} gives background on rankings and distances. Section~\ref{sec:mm} details the Mallows models for partial permutations and shows efficient ways of dealing with them, sampling, or computing statistics. Section~\ref{sec:measure} shows how to separate both subpopulations of concentric MM. Section~\ref{sec:uBorda} addresses the problem of the estimation of the consensus ranking. Finally, Section~\ref{sec:experiments} details the experimental evaluations and Section~\ref{sec:conclusions} concludes the paper.

\section{PRELIMINARIES}\label{sec:perliminaries}
The group of permutations of $n$ items is denoted $S_n$.
The identity permutation is $e=1,2,\ldots,n$, the group operation is the composition $\sigma\cdot\pi$, denoted $\sigma\pi$, and the inverse of $\sigma$ is denoted $\sigma^{-1}$. We consider that permutation $\sigma$ represents a ranking of items, where  $\sigma(i)$ is the rank of item $i$. 


Every permutation $\sigma\in S_n$ can be uniquely represented by its \textit{inversion vector} ${\mathbf V(\sigma)=(V_1(\sigma ), \ldots, V_{n-1}(\sigma ))}$, where\footnote{We have chosen one of the different equivalent definition of inversion vectors~\cite{gMallows,Mandhani2009}. }
\begin{equation}
 V_j(\sigma) = \sum_{i>j} \mathbbm{I}[\sigma (i) < \sigma (j)] 
,
 \label{eq:kendall_decomp}
\end{equation}
$\mathbbm{I}$ is the indicator function and $0\leq  V_j(\sigma) < n-j$. There exists a quasi-linear time complexity bijection between each possible inversion vector and permutations in $S_n$~\cite{McClellan1974}. 

We consider the Kendall's-$\tau$ distance,  $d(\sigma ,\pi)$, which counts the number of pairwise disagreements between $\sigma$ and $\pi$. We use $d(\sigma)$ to denote $d(\sigma,e)$. The relation between the Kendall's-$\tau$ distance and inversion vectors comes from the fact that $d(\sigma)=\sum_j V_j(\sigma)$. For top-$k$ rankings we use a generalization of the Kendall's-$\tau$ distance that assumes that items that cannot be compared do not increase the distance. This is equivalent to the generalization of  \cite{Fagin2003} with the $p$ parameter equal to 0. 


We consider the Mallows model (MM) to model distributions on $S_n$. MM expresses the probability of ranking $\sigma\in S_n$ as $p(\sigma) \propto \exp(-\theta d(\sigma, \sigma_0))$. We will make use of the convenient observation made in previous paragraphs that claims that $d(\sigma) = \sum_j V_j(\sigma)$ to rewrite the MM as follows~\cite{meila07},

\begin{align}
\begin{aligned}
\label{eq:mm}
& p(\sigma) = \frac{\prod_{j=1}^{n-1}\exp(-\theta V_j(\sigma \sigma_0^{-1}))}{\psi_{n}} \\ & \text{where} \quad \psi_{n} = \prod_{j=1}^{n-1}  \psi_{n,j} = \prod_{j=1}^{n-1}  \frac{1-\exp(-\theta (n-j+1))}{1-\exp(-\theta)}. 
\end{aligned}
\end{align}

We denote as $M(\sigma_0, \theta)$ a MM with consensus ranking (or location parameter)  $\sigma_0$ and with dispersion parameter $\theta$. A random permutation distributed according to this model is denoted as $\sigma \sim M(\sigma_0, \theta)$. Since the Kendall's-$\tau$ distance is right invariant, we can assume that $\sigma_0=e$ w.l.o.g. Interestingly, and base of our sampling algorithm, is that the probability distribution of each item in the inversion vector $\mathbf V(\sigma\sigma_0^{-1})=(V_1(\sigma\sigma_0^{-1} ), \ldots, V_{n-1}(\sigma\sigma_0^{-1} ))$ for $\sigma \sim M(\sigma_0, \theta)$ can be expressed as 
\begin{equation}
p(V_j(\sigma\sigma_0^{-1}) = r) = \frac{\exp(-\theta r)}{\psi_{n,j}}.
\label{eq:prob_V}
\end{equation}

It follows that $p(\sigma)$ can be stated as the product of independent factors, $p(\sigma)=\prod_j p(V_j(\sigma\sigma_0^{-1}))$,~\cite{Mandhani2009}. 

Learning the maximum likelihood estimate (MLE) of a MM given a sample of permutations is done in two stages~\cite{Mandhani2009}. Firstly, the MLE of the central ranking is the Kemeny ranking of the sample~\cite{Ali2011}. Since this problem is NP-hard~\cite{Dwork:2001:RAM:371920.372165}, usually the Borda ranking is used. Borda can be computed in quasi-linear time and is an unbiased estimator of the Kemeny ranking of a sample distributed according to Mallows
~\cite{Fligner1988}. Secondly, the MLE of the dispersion parameters are obtained numerically~\cite{Irurozki2016b}.


Mixture models are used to combine different simple probability models to model large, heterogeneous populations. In our motivating problem, we consider a heterogeneous population of two subpopulations: one made of expert voters and another one of non-expert voters. This is modeled with a mixture of two concentric components where the probability of each permutation is

\begin{equation}
p(\sigma) = r \frac{\exp(-\theta_g d(\sigma, \sigma_0)) }{\psi_n} + (1-r) \frac{\exp(-\theta_b d(\sigma, \sigma_0))}{\psi_n}.
\end{equation}

where there is one consensus ranking $\sigma_0$, a dispersion parameter of the expert voters, $\theta_g$, a dispersion of the non-expert ones, $\theta_b<\theta_g$ and a proportion of expert voters in the population $r$, denoted mixture parameter. We denote mixtures in which all the components have the same central ranking as \textit{concentric}. 

\section{TOP-$k$ RANKING STATISTICS UNDER THE MALLOWS MODEL}\label{sec:mm}

In this section, we study the problems of sampling top-$k$ rankings from a MM, sampling linear extensions of top-$k$ rankings, and computing the probability of a top-$k$ ranking efficiently.

We start with the primary problem in statistics: Computing the probability of a top-$k$ {ranking $\sigma$} efficiently. 
The probability of $\sigma$ is the sum of the probabilities of all its linear extensions, ${p(\sigma) = \sum_{\sigma'\in L(\sigma)} p(\sigma')}$, so a naive approach computes $p(\sigma)$ in $\mathcal{O}((n\log n)(n-k)!)$. We propose a $\mathcal{O}(n+k\log k)$ expression to compute $p(\sigma)$ in the next lemma.

\begin{lem}\label{thm:proba}
The probability of the top-$k$ ranking $\sigma$ is $${p(\sigma) =  \exp(-\theta d(\sigma, \sigma_0) ) \frac{ \psi_{n-k,\theta}} {\psi_{n,\theta}}}.$$
\end{lem}


The complexity of the previous expression comes from the normalization constant ($\mathcal{O}(n)$, in Equation~\eqref{eq:mm}) and the computation of the distance ($\mathcal{O}(k \log k)$). A proof can be found in the appendix.

\paragraph{Sampling}
When we consider the problem of sampling complete rankings (instead of top-$k$ rankings),  RIM~\cite{Regenwetter2004} offers a convenient alternative. RIM samples a ranking in two steps: First, it samples vector $\mathbf R(\sigma) = (R_1, \ldots, R_n)$, where $1 \leq R_i \leq i$ and $p(R_i)$ for $\sigma\sim M(\sigma_0, \theta)$ is known~\cite{Regenwetter2004}. Second, starting with an empty vector $\sigma$ and letting $i$ range in $[1,n]$, it inserts item $i$ in position $R_i$ of $\sigma$, shifting backwards items if necessary. Due to this shifting strategy, $\sigma(i)$ is not known until the last iteration for every $i\leq n$. This means that the only way of sampling a top-$k$ ranking with RIM is to sample a complete ranking and then to censor it. This is clearly a terrible idea, especially when $k<<n$. There exits an improvement for this process with complexity $\mathcal{O}(k^2 4^k + k^2\log n)$~\cite{NIPS2018_7691}.

In the following lines, we introduce a new sampling algorithm for top-$k$ rankings with quasi-linear complexity. Our proposed algorithm can also be used to sample complete rankings by setting $k=n$ and its complexity improves over RIM's. Also, code for sampling is attached as supplementary material and will be made available at public repositories upon publication. 

Algorithm~\ref{alg:sampling} samples a top-$k$ ranking ${\sigma\sim M(\sigma_0, \theta)}$ in  time $\mathcal{O}(k \log k)$ and memory $\mathcal{O}(k)$. It is based on the following results: First, Equation~\eqref{eq:prob_V} gives the probability of each position of the inversion vector independently $\mathbf V(\pi\sigma_0^{-1})$: sampling the first $k$ positions is linear in $k$. Second, it generates the partial permutation $\pi\sigma_0^{-1}$ from the inversion vector with a quasi-linear time bijection~\cite{McClellan1974}. Finally, the top-$k$ ranking $\sigma$ is distributed as $\sigma\sim M(\sigma_0, \theta)$ where $\sigma = \pi^{-1}$ \cite{mallows}.


  \begin{algorithm}[h]
  \caption{Sample top-$k$ in $O(k \log k)$}
  \label{alg:sampling}
\KwData{$n$, $k$, $\theta$, $\sigma_0$}
\KwResult{$\sigma$: Top-$k$ ranking of $n$ items distributed according to $M(\sigma_0, \theta)$}
 initialization\;
 \For{$j\in[1,k]$}{
 $V_j(\pi\sigma_0^{-1}) =$ random choice in $[n-j]$ with choice probabilities of Eq.~\eqref{eq:prob_V}\;
  $\pi\sigma_0^{-1}$ =  transform  $V(\pi\sigma_0^{-1})$ with the bijection in~\cite{McClellan1974} \;
  \Return $\pi^{-1}$
 }
  \end{algorithm}

\paragraph{Sampling linear extensions}
A similar approach is used to sample a linear extension of a top-$k$ ranking in  $\mathcal{O}((n-k) \log (n-k))$, we refer to the supplementary material for the pseudo-code. It follows directly from the previous result. Sampling a linear extension is done by sampling $V_j(\sigma)$ from $j\geq k$. Then, obtaining $\sigma$ is $\mathcal{O}((n-k) \log (n-k))$~\cite{McClellan1974}.

Expressions for the expected distance and variance are included in the appendix for completeness and also appear in~\cite{busa2019optimal}.

\section{IDENTIFIABILITY OF CONCENTRIC MIXTURES}
Identifiability of concentric mixtures has been claimed to be the most difficult scenario in mixture identifiability~\cite{Chierichetti2015}. Indeed, the concentric mixtures of Normal components, which belongs to the same family as the MM, are known to be non-identifiable. 

Identifiability is guaranteed if we can (i) recover the ground truth ranking and (2) separate the rankings in each population. Each of these points is addressed in the following sections to claim identifiability for concentric mixtures.

\subsection{Provably separation of the sub-populations of each component}\label{sec:measure}

In this section we consider the problem of separating the rankings of the two components of a concentric mixture of MM. We propose an algorithm that  can separate the two sub-populations under mild conditions of the separation of the dispersion parameters among the mixture components and the mixture parameter $r$. Our proposed algorithm is based on finding the separation of the mean distance of each top-$k$ ranking $\sigma$ to all the others in the sample, which is defined as follows, 
\begin{equation} 
    \delta_{\sigma} = \frac{1}{|S|-1}\sum_{\sigma'\in S \setminus \lbrace \sigma \rbrace} d(\sigma, \sigma').
    \label{eq:weight}
\end{equation}

Recall that the expected distances $\mathbb{E}[d(\sigma, \sigma_0)]$ and $\mathbb{E}[\delta_{\sigma}]$ for $\sigma\sim M(\sigma_0, \theta)$ only depend on $\theta$. 
We will use the following observation along this section~\cite{Korba2017}:
The expected distance between a random Mallows ranking $\sigma \sim M(\sigma_0, \theta)$ and the consensus ranking $\sigma_0$, $\mathbb{E}[d(\sigma_0,\sigma)]$, is bounded by the expected pairwise distance of two random Mallows permutations, $\sigma, \sigma' \sim M(\sigma_0, \theta)$ as follows. 

\begin{equation}
\frac{1}{2}\mathbb{E}[d(\sigma',\sigma)] \leq \mathbb{E}[d(\sigma_0,\sigma)] \leq \mathbb{E}[d(\sigma',\sigma)].
\label{eq:bound_clemencon}
\end{equation}

First, we introduce some notations recalling our motivating example: a concentric mixture represents a population of voters with two homogeneous sub-populations, one of expert voters distributed according to $M(\sigma_0, \theta_g)$ and one of non-expert voters $M(\sigma_0, \theta_b)$ for $\theta_g>\theta_b$. Let $\beta\sim M(\sigma_0,\theta_b)$ and $\gamma\sim M(\sigma_0,\theta_g)$ be two random top-$k$ rankings, i.e., $\beta$ is a vote cast by a non-expert (bad) voter and $\gamma$ a vote cast by an expert (good) voter. In this section we show that $\mathbb{E}[\delta_\beta]$ and $\mathbb{E}[\delta_\gamma]$ are well separated and give an algorithm that separates the two sub-populations in ${\mathcal{O}(m^2)}$ time provided that the components are sufficiently far from each other.

Now, we show an auxiliary lemma that bounds the expected distance between a random expert $\gamma$ ranking and a random non-expert $\beta$ ranking. 
\begin{lem}
\label{lem:upper_bound_mixed}
The expected distance between two rankings of different components $ \mathbb{E}[d(\beta,\gamma)]$ is bounded as follows 
\begin{equation}
\mathbb{E}[d(\beta,\sigma_0)] \leq 
\mathbb{E}[d(\beta,\gamma)] \leq \mathbb{E}[d(\beta',\beta)].
\end{equation}
\end{lem}

The following result shows that the expected mean distances of the expert and non-expert top-$k$ rankings are well separated. 

\begin{thm}\label{thm:2pops}
Let $M(\sigma_0, \theta_g)$ and $M(\sigma_0, \theta_b)$ be the two components of a concentric mixture of top-$k$ rankings in which $\mathbb{E}[d(\beta, \sigma_0)] \ge c \cdot \mathbb{E}[d(\gamma, \sigma_0)]$ for $c \ge 2$. Let $r \in [0,1]$ be the mixture parameter. 
The expected mean distance between rankings of different components $\beta\sim M(\sigma_0,\theta_b)$ and $\gamma\sim M(\sigma_0,\theta_g)$ is bounded as follows,

\begin{equation}
\mathbb{E}[\delta_{\beta}-\delta_{\gamma}] \ge \mathcal{O}( c \cdot r \cdot \mathbb{E}[d(\gamma, \sigma_0)]).
\end{equation}
\end{thm}

Theorem~\ref{thm:2pops} suggests that a clustering algorithm in $\delta_{\sigma}$ for every $\sigma\in S$ can separate the population of expert voters from non-expert voters. We show that, indeed, a single linkage clustering algorithm can separate the sub-populations with proven guarantees. 

\begin{thm}\label{thm:clustering}
Let $M(\sigma_0, \theta_g)$ and $M(\sigma_0, \theta_b)$ be the two components of a concentric mixture of top-$k$ rankings in which $\mathbb{E}[d(\beta, \sigma_0)] \ge c \cdot \mathbb{E}[d(\gamma, \sigma_0)]$ for $c > 2$. Let $r \in (0,1]$ be the mixture parameter. 
There exists an algorithm that separates the samples from both components with probability $1- \epsilon$ in $\mathcal{O}(m^2)$ when the number of samples is at least 
\begin{equation}
    m > \left( \frac{n(n-1)}{ (c-2) \cdot r \cdot \mathbb{E}[\gamma, \sigma_0] } \right)^2 \frac{\log(2/\epsilon)}{2}.
\end{equation}
\end{thm}


Hence we have shown that under certain conditions about the expected distances of both populations, we can separate both components with high probability. This would allow us, given that we also know the central permutation, to learn $\theta_g$ and $\theta_b$ dispersion parameters, using the maximum likelihood estimation as described in \cite{Irurozki2014}, adapted to the case of top-$k$ rankings here.

Moreover, while we need the expected distances of both populations to be separated, we can deal with cases where one of the distribution is very close to uniformity, which can be a problematic case for some approaches~\cite{liu2018efficiently}. 

Finally, computing $\delta_{\sigma}$ for every permutation in the sample requires $\mathcal{O}(m^2)$ distance calculations. However, in practice, we can approximate this value with high probability, reducing the number of distance computations using a result from a corollary of Hoeffding's bound. 

\subsection{Estimating the consensus ranking of top-$k$ rankings}\label{sec:uBorda}

In this section we study the following problem: given a sample of top-$k$ rankings distributed as a concentric mixture of MM, find the MLE of the central ranking. Since all the components of a concentric mixture have the same consensus ranking $\sigma_0$, this problem boils down to a rank aggregation problem, as in the single-component case: the MLE is exactly given by the Kemeny ranking. Unfortunately, computing this ranking is NP-hard for $n>4$,
~\cite{Dwork:2001:RAM:371920.372165}. 

For a sample of complete rankings drawn from a MM, Borda is an approximation to the Kemeny ranking~\cite{Fligner1988}. Moreover, Borda is quasi-linear in time and outputs the correct $\sigma_0$ w.h.p. with a polynomial number of samples~\cite{Caragiannis2013}. However, there is no quality result for the case in which the sample consists of top-$k$ rankings i.i.d. from a MM. In the next lines, we provide a sample complexity for Borda over top-$k$ rankings from a MM and then extend it to the case of concentric mixtures.

A crucial difference between complete and top-$k$ rankings drawn from a MM\footnote{We assume, as in previous sections and w.l.o.g., that $\sigma_0=e$.}, is that in top-$k$ Mallows rankings $\sigma$ the probability of observing item $i$, i.e., $p(\sigma(i)\leq k)$, decreases with $i$. Intuitively, this means that Borda will need fewer samples to \textit{guess the rank} of smaller $i$'s. We formalize this intuition in the next result: We bound the number of samples that Borda requires to rank items $i$ and $i+1$ in the correct order with probability $1-\epsilon$.

For this analysis, we first need the following definition. 

\begin{defin}\label{thm:delta}
Let $i\in[0,n-1]$, $k\in[0,n]$.
\begin{equation}
\Delta^{ik} =  \sum_{\sigma:\sigma(i) \leq k}p(\sigma) - \sum_{\sigma:\sigma(i+1)\leq k}p(\sigma).
\end{equation}
\end{defin}

Despite in general there is no closed-form expression for $\Delta^{ik}$, we can give a convenient expression for the case where $i=1$, $\Delta^{1k}$.

\begin{lem} \label{thm:min_Delta}
The minimum value for $\Delta^{ik}$,  $\Delta^{1k}=\min_{i}\Delta^{ik}$ can be computed in $O(k^2+kn)$ as follows: 
\begin{align}
\begin{aligned}
\Delta^{1k} = 
 & \sum_{r_1=0}^{k-1}  \frac{\exp(-\theta r_1) }{\psi_{n,1}} \\
 & - \sum_{r_1=0}^{k-1} \sum_{r_2=0}^{k-2} \frac{\exp(-\theta r_1) \exp(-\theta r_1)}{\psi_{n,1}\psi_{n,2}} \\
 & - \sum_{r_1=k}^{n} \sum_{r_2=0}^{k-1}  \frac{\exp(-\theta r_1) \exp(-\theta r_2)}{\psi_{n,1}\psi_{n,2}}.
\end{aligned}
\end{align}
\end{lem}

Based on the above results, we can bound the sample complexity of Borda for giving the correct ordering of items $i$ and $i+1$. 

\begin{thm}\label{thm:sample_complexity_expertness}
Let $S$ be a sample of top-$k$ rankings drawn from a MM with dispersion $\theta$. Borda for sample $S$ orders the pair of items $i$ and $i+1$ correctly with probability $1-\epsilon$ when the number of rankings in the sample is at least
\begin{equation}
\begin{split}
m  \geq \mathcal{O} \bigg(  n^2 \log \epsilon^{-1} \Big (\frac{ k^2 (1-\exp(-\theta))^2}{1-\exp(-\theta n)} - i \Delta^{1k} \Big ) ^{-2} \bigg). \\
\end{split}
\label{eq:sample_complexity}
\end{equation}
\end{thm}

The above theorem formalizes the intuitive idea that the sample complexity of Borda to order correctly items $i$ and $i+1$ on a sample of top-$k$ rankings increases polynomially with $i$. Therefore, this bound generalizes to the sample complexity of Borda to obtain the correct ranking w.h.p. for a sample of top-$k$ drawn from a mixture of MM with parameters $\sigma_0$, $\theta_g$ and $\theta_b$.
\begin{cor}
Borda returns the correct central ranking $\sigma_0$ when the number of samples in the non-expert component satisfies Equation~\eqref{eq:sample_complexity} with $i=n$ and $\theta=\theta_b$. 
\end{cor}


\section{EXPERIMENTS}\label{sec:experiments}

In this section, we validate empirically our proposal. The experimental framework is as follows. In the first two experiments, we generate a sample of partial rankings, using Algorithm~\ref{alg:sampling}, with parameters $n=30$ and $k=10$, from a mixture of concentric MM, both centered at a random $\sigma_0$ and with two dispersion parameters, $\theta_b, \theta_g$. The ratio of experts voters is denoted $r$. 

\subsection{Voters separability}\label{sec:clustering_experiment}

The goal of this experiment is to separate the two components of rankings based on the mean distance of each of the rankings to all the others as defined in Equation~\eqref{eq:weight}. To compute these distances, we used the generalization of Kendall's-$\tau$ distance described in Section~\ref{sec:perliminaries}.

The experimental framework is as follows: 
\begin{itemize}
    \item $m_g = 40$ rankings from a $M(\sigma_0, \theta_g)$ such that $\mathbb{E}[d(\gamma,\sigma_0)] \in \{3, 8, 13, \ldots, 48\}$
    \item $m_b = 60$ rankings from a $M(\sigma_0, \theta_b)$ such that $\mathbb{E}[d(\beta, \sigma_0)] = c \cdot \mathbb{E}[d(\gamma,\sigma_0)]$ with $40 > c \ge 3$ and $\mathbb{E}[d(\gamma,\sigma_0)] \le 217$ (bound corresponding to the uniform distribution).
\end{itemize}

For a sample distributed as the model detailed above, we compute each $\delta_{\sigma}$ and apply a 2-means clustering. The error is measured as the percentage of badly-separated voters in the sample. 

We can observe the results of this experiment in Figure~\ref{fig:foobar} (a). As we can see, as the ratio $c$ increases, the separation gets more accurate. Indeed, for $c \ge 9$ the percentage of wrongly separated rankings is almost always below $10\%$ and is very close to $0$ when non-expert voters are close to the uniform distribution. These results are consistent with Theorem~\ref{thm:2pops}.

\subsection{Consensus estimate with expert Borda}

\begin{figure*}
    \centering
    \subfigure[]{\includegraphics[align=c,width=0.3\textwidth]{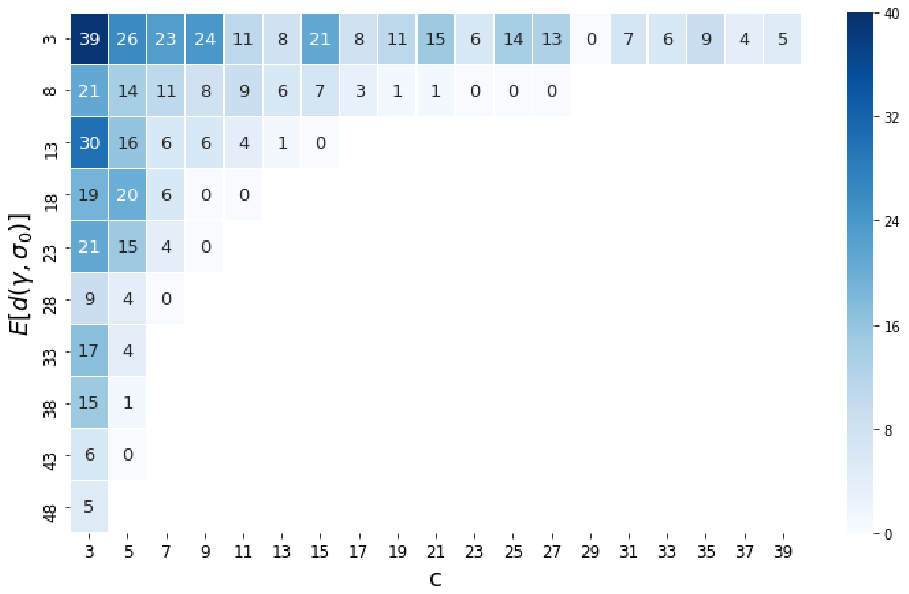}}    \subfigure[]{\includegraphics[align=c,width=0.3\textwidth]{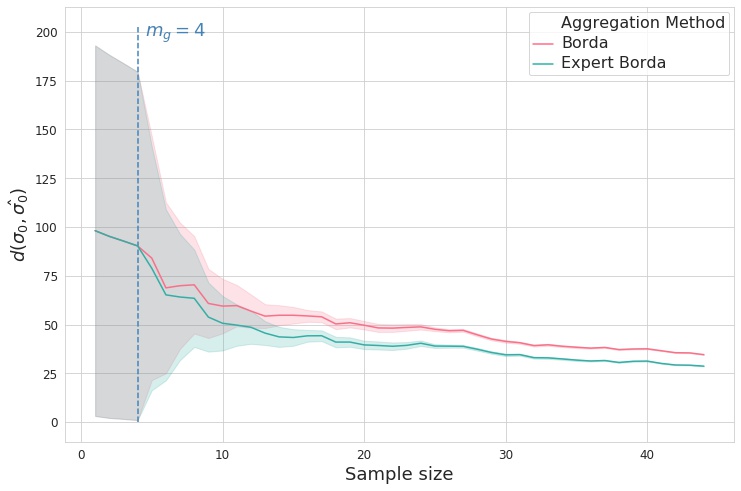}} 
    \subfigure[]{\includegraphics[align=c,width=0.3\textwidth]{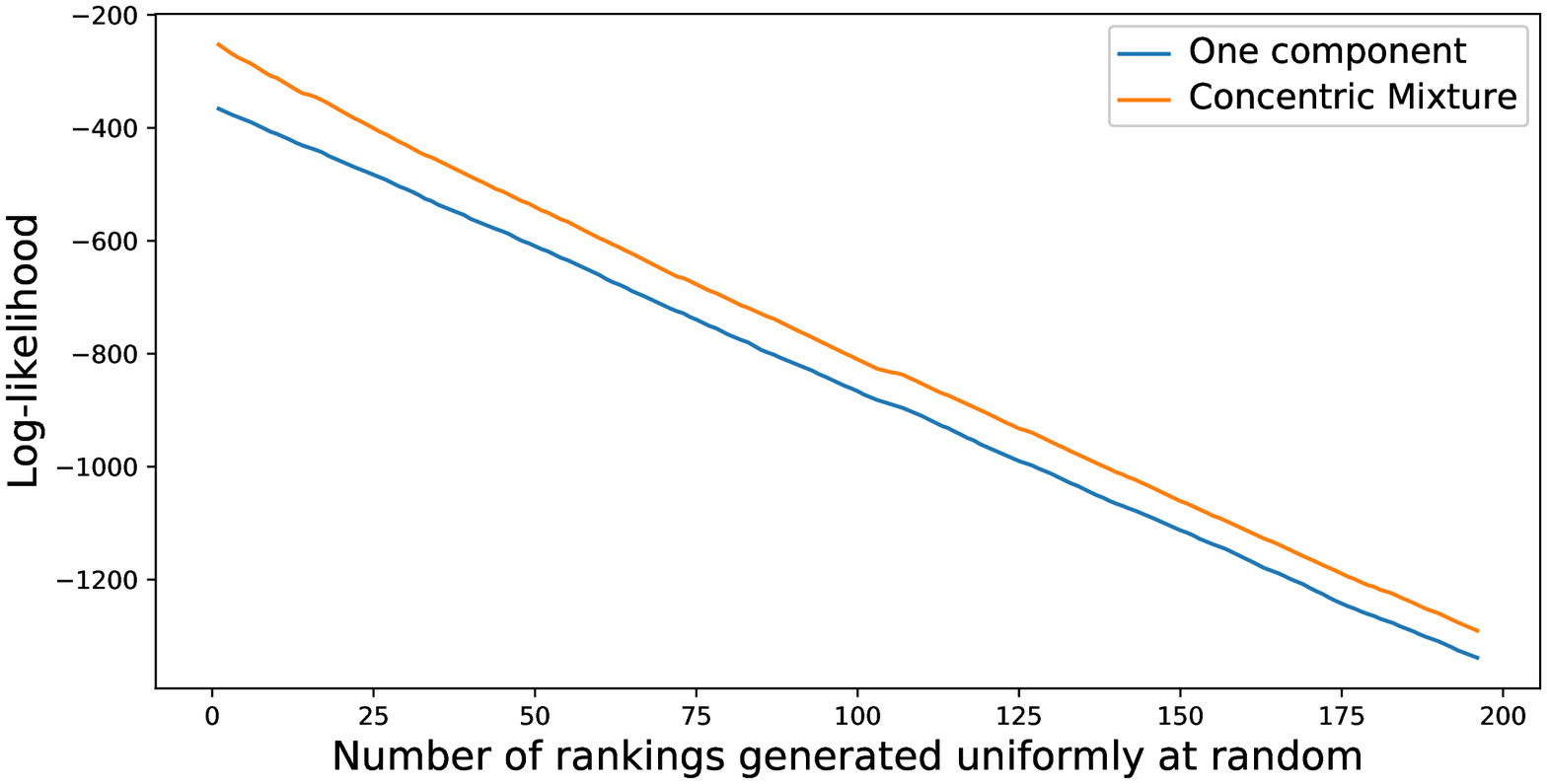}} 
    \caption{(a) Total error (\%) of separation between expert and non-expert voters, using K-means. (b) Distances between the consensus and its estimates, with different sizes of samples and two aggregation methods. (c) Comparison of the log-likelihoods of a distribution of rankings, considering it as a single component of Mallows Model or as a concentric mixture of Mallows Models with different dispersion parameters}
    \label{fig:foobar}
\end{figure*}

Our goal here is to estimate $\sigma_0$ for a sample drawn from a mixture of two concentric Mallows. We show that, with this type of population, there exists a better alternative to Borda, using the separability of the voters. 
    
Indeed, we build this alternative based on the following observations:
(1) The sample complexity is smaller for expert voters than for non-experts and 
(2) we can identify expert voters with a high probability as shown in Section~\ref{sec:measure}. 
With this in hand, we propose an aggregation method that (1) identifies expert voters (2) constructs a top-$k$ ranking by aggregating experts' rankings, and (3) fills the missing $n-k$ positions with the data of the whole sample of rankings.
We denote this approach expert Borda ($eBorda$) and show its efficiency in the present experiment. 

Hence, we sample a population of voters following a mixture of two Mallows model. We will take $m_g=4$ expert voters from $M(\sigma_0,\theta_g)$, and $m_b = 40$ non-expert voters from $M(\sigma_0,\theta_b)$. $\theta_g$ and $\theta_b$ are chosen such that $\mathbb{E}[d(\sigma_0,\gamma)]=10$ and $\mathbb{E}[d(\sigma_0,\gamma)]= 75$.

The estimate $\hat{\sigma_0}$ for $\sigma_0$ is computed with both the Borda method and with our proposed $eBorda$, using the same growing sample, with size $\{1, 2, 3, ..., 44\}$: first 1 ranking, then 2, ... starting with the experts' rankings. For each number of rankings, we measure the error of the estimate, $d(\sigma_0, \hat{\sigma_0})$, as the mean between the maximum and the minimum distance the estimate could have to the consensus if all its positions were filled. This is repeated ten times and average values for the distances are given. 

The results are given in Figure~\ref{fig:foobar} (b), where the $x$-axis indicates the number of rankings considered for the estimation, and the $y$-axis gives the error of estimation. The vertical line marks the step from which non-expert rankings are added to the sample. Finally, the aggregation of top-$k$ rankings results on a single top-$k'$ ranking where $k \le k' \le n$. When measuring our top-$k'$ estimate to our complete $\sigma_0$ consensus, not all the pairs can be compared. The shadow around the curve is the bound on the distances between any linear extension of the partial estimate and the consensus. Hence the larger $k'$, the smaller the shadow.   

As expected, we can observe that $eBorda$ can perform better than Borda. Indeed, the separation of both components allows us to keep a more accurate estimate for the first $k$ positions, using only the expert rankings to estimate it. Nevertheless using the non-expert rankings afterward allows us to complete the estimate into a full ranking, making the uncertainty of the error decrease faster as we can see with the shadows around the curves narrowing. 

\subsection{Real data example}

To test the identifiability on real data, we used a dataset already used in~\cite{gMallows}, for which 98 college students were asked to rank five words according to its strength of association with the word ``idea". The five words to classify were: (1) thought, (2) play, (3) theory, (4) dream, and (5) attention. These were to be ranked from $1$ to $5$, $5$ being the most strongly associated with the target word. In our present example, $m=98$ and $n=5$.  

We assumed the dataset to be distributed according to a Mallows Model and estimated its Maximum Likelihood Estimates, $\sigma_0 = 5, 1, 4, 3, 2$ and $\theta_g = 1.43$. We then simulated a sample of $2 \cdot m$ raters generated uniformly at random.

Simulated raters were added, one by one, and at each step, two different models were fitted: (1) a MM and (2) a mixture of concentric MM, for which each component is determined performing a $2-Means$ clustering, using the same procedure as in Section~\ref{sec:clustering_experiment}. Then, we compute the log-likelihood of each model. The results are represented in Figure~\ref{fig:foobar} (c), we can see that the mixture of concentric MM fits the data better, even in this case in which the number of random is large, larger than $n!$.

\section{CONCLUSIONS}\label{sec:conclusions}
In this paper, we have studied the alleged most difficult setting in the learnability of mixtures of location-scale distributions: the case in which the location parameters are the same. We denote this case as concentric and focus on the Mallows model for top-$k$ rankings. This situation arises in the case where we have two populations of voters, one expert in the field of the rating and the other non-expert. 

We have proposed a $\mathcal{O}(k \log k)$ sampling algorithm for top-$k$ ranking, which dramatically reduces the requirement of the samplers in the literature. 

We have also proposed an algorithm for the learnability of the parameters of the concentric mixture of top-$k$ rankings with a high probability in polynomial time. It is based on our two following results: We have bounded the sample complexity of the Borda algorithm to recover the ground truth consensus ranking. And second, we have been able to separate the rankings of each component in polynomial time with high probability. 

Interesting extensions to our work could be to generalize our results to concentric Mallows mixtures of more than two components, non-concentric mixtures of Mallows model, or other models such as Plackett-Luce's model. 

\bibliographystyle{plain}
\bibliography{mendeley}

\end{document}


%

%

\onecolumn
\aistatstitle{Instructions for Paper Submissions to AISTATS 2021: \\
Supplementary Materials}
\section{Proof of Lemma~\ref{thm:proba}}

\begin{proof}
Let the linear extensions of $\sigma$ be $L(\sigma)$. The sum of their probabilities is 

\begin{equation}
\begin{split}
p&(\sigma) = \sum_{\sigma'\in L(\sigma)} p(\sigma') 
= \sum_{\sigma'\in L(\sigma)}  \frac{ \prod_{j=1}^{n-1} \exp(-\theta V_j(\sigma'))}{\psi_{n,\theta}} \\
&= \sum_{\sigma'\in L(\sigma)}  \frac{ \prod_{j=1}^{k} \exp(-\theta V_j(\sigma))  \prod_{j=k+1}^{n-1} \exp(-\theta V_j(\sigma'))}{\psi_{n,\theta}} \\
&=  \frac{  \exp(-\theta d(\sigma)) \sum_{\sigma'\in L(\sigma)} \prod_{j=k+1}^{n-1} \exp(-\theta V_j(\sigma'))}{\psi_{n,\theta}} \\
&= \frac{ \exp(-\theta d(\sigma)) \psi_{n-k,\theta} }{\psi_{n,\theta}},
\end{split}
\end{equation}
where $\psi_{k,\theta}$ is defined in Equation (2). The overall complexity is dominated by $\psi_{n,\theta}$, which is $\mathcal O(n)$
\end{proof}

\section{Sampling linear extensions}
 
  \begin{algorithm}[H]
    \caption{Sampling linear extensions in $\mathcal{O}((n-k) \log (n-k))$}
    \label{alg:linear_extension}
    \KwData{$n$, $k$, $\theta$, $\sigma'$}
    \KwResult{$\sigma$: Full ranking}
    $V_j(\sigma) =$ bijection from $\sigma'$\;
    \For{$j\in[k+1,n]$}{
     $V_j(\sigma) =$ random choice in $[n-j]$ with choice probabilities of Eq. (3)\;
      $\sigma$ = transform  $V(\sigma)$ with the bijection in~\cite{McClellan1974}\;
      \Return $\sigma^{-1}$
     }
  \end{algorithm}

Along the section, we have made use of the following result.


\begin{lem} \label{thm:inverse}
Let $\sigma\in S_n^k$ where $\sigma \sim M(\sigma_0, \theta)$. Then, $\sigma^{-1}$ is a top-$k$ ranking distributed according to the same distribution, $\sigma^{-1} \sim M(\sigma_0, \theta)$, and $d(\sigma) = d(\sigma^{-1}) $.
\end{lem}
\begin{proof}
Let $\sigma\sim MM(e,\theta)$ and $\pi=\sigma^{-1}$. Note that for $\sigma\in S_n^k$  then $\pi$ is a top-$k$ ranking. Moreover, $d(\sigma)=d(\pi)$ and, since the MM is defined upon the distance function, it follows that $p(\sigma)=p(\pi)$ for every $\sigma$ and therefore $\pi\sim MM(e,\theta)$. In the case that $\sigma_0\neq e$, taking the right invariance property of the Kendall's-$\tau$ distance, it follows that $\pi\sigma_0\sim MM(\sigma_0,\theta)$. Finally, if $\pi\sigma_0$ is a top-$k$ ranking, then $(\pi\sigma_0)^{-1}$ is a top-$k$ list, which concludes the proof. 
\end{proof}

\section{Proof of Equation~\ref{eq:bound_clemencon} of the paper}

\begin{proof}
Let $p_{ij}$ be the marginal probability that item $i$ is preferred to item $j$: \begin{equation}\label{eq:marginal_prob}
p_{ij} = \sum_{\sigma:\sigma(i)<\sigma(j)} p(\sigma).
\end{equation}
The exact expression can be found in \cite{DBLP:conf/icml/Busa-FeketeHS14} but for the proof, we just need to highlight that $p_{ij} = 1-p_{ji}$. This pairwise comparison expression and the assumption that $\sigma_0=e$  lets us rewrite the expected distance from the mode as follows
\begin{equation}
\mathbb{E}[d(\sigma, \sigma_0)] = \sum_{i<j}p_{ji}. 
\end{equation}

The expected pairwise distance can be written as follows
\begin{equation}
\mathbb{E}[d(\sigma,\sigma')] = 2*\sum_{i<j}p_{ij}p_{ji}= 2*\sum_{i<j}p_{ji}-p_{ji}^2.
\end{equation}

With this restatement, the bound can be easily proved. 
\begin{equation}
\sum_{i<j}p_{ji}-p_{ji}^2 \leq  \sum_{i<j}p_{ji} \leq 2*\sum_{i<j}p_{ji}-p_{ji}^2.
\end{equation}

Note that this result holds for partial permutations as well. 
\end{proof}

\section{Expected distance $\mathbb{E}[D]$ and variance $\mathbb{V}[D]$ under the Mallows model}

\begin{lem}\label{thm:expected}
Let $D$ be a random variable defines as $D=d(\sigma, \sigma_0)$ for a random Mallows ranking $\sigma$. The expectation and variance of $D$ are as follows:


\begin{equation}
\begin{split}
& \mathbb{E}[D] = \frac{k \cdot \exp(-\theta)}{1-\exp(-\theta)} - \sum_{j=n-k+1}^{n}  \frac{j \exp(-j \theta)}{1-\exp(-j\theta)},\\
& \mathbb{V}[D] =  \frac{k \cdot \exp(-\theta)}{(1-\exp(-\theta))^2} - \sum_{j=n-k+1}^{n}  \frac{j^2 \exp(-\theta j)}{(1-\exp(-\theta j))^2}. \\
\label{eq:expectation}
\end{split}
\end{equation}
\end{lem}

\begin{proof}
The moment generating function $ M(t)=\mathbb {E} \left[\exp(tD)\right]$ of the distance $D = d(\sigma_0, \sigma)$ of a random Mallows permutation $\sigma$ can be factorized in this way \cite{gMallows}:
\begin{equation}
\label{eq:mgf}
M(t) = \prod_j M_j(t) = \prod_j \frac{1 - \exp(t(n-j+1))}{(n-j+1)(1-\exp(t))}.
\end{equation}

It's derivative, is as follows :


\begin{equation}
\frac{d\ln M_j(t)}{dt} = \frac{\exp(t)}{1-\exp(t)} - \frac{(n-j+1)\exp(t(n-j+1))}{1-\exp(t(n-j+1))}.
\end{equation}

For exponential models as the MM, expected values and variances can be easily written as function of the moment generating function.

\begin{equation}
\begin{split}
\mathbb{E}[V_j] &= \frac{d \ln M_j(t)}{dt} \Big|_{t=-\theta} \\
&= \frac{\exp(-\theta)}{1-\exp(-\theta)} - \frac{(n-j+1)\exp(-\theta(n-j+1))}{1-\exp(-\theta(n-j+1))} 
\end{split}
\end{equation}

\noindent 
and 

\begin{equation}
\begin{split}
\mathbb{V}[V_j]  &= \frac{d^2 \ln M_j(t)}{dt^2} \Big|_{t=-\theta}  \\
&= \frac{\exp(-\theta)}{(1-\exp(-\theta))^2} - \frac{(n-j+1)^2 \exp(-\theta(n-j+1))}{(1-\exp(-\theta(n-j+1)))^2}. 
\end{split}
\end{equation}

The proof concludes by noting that $\mathbb{E}[D]  = \sum_{j=1}^k \mathbb{E}[V_j] $ and $\mathbb{V}[D]  = \sum_{j=1}^k \mathbb{V}[V_j] $.

\end{proof}

\section{Proof of Lemma~\ref{lem:upper_bound_mixed}}

\begin{proof}
We start by the right hand side of the equation:
Let $g_{ij}$ and $b_{ij}$ be the marginal probabilities for good and bad raters respectively, as defined in Equation \eqref{eq:marginal_prob}. 

Let us assume that $\forall i < j$, $b_{ij} \ge g_{ij}$ and using Corollary 3 from \cite{DBLP:conf/icml/Busa-FeketeHS14} we even have that $\forall i <j$, $1 \ge b_{ij} \ge g_{ij} > \frac{1}{2}$. Then $\forall i<j$, $\exists \epsilon_{ij} \in [0, \frac{1}{2}]$ such that $b_{ij} = g_{ij} + \epsilon_{ij}$.

We can rewrite the expected distances as functions of the marginal as follows:
\begin{itemize}
    \item $\mathbb{E}[d(\beta, \gamma)] = \sum_{i < j} b_{ij} + g_{ij} - 2b_{ij}\cdot g_{ij}$
    \item $\mathbb{E}[d(\beta,\beta')] = \sum_{i<j} 2 b_{ij} - 2b_{ij}^2$
\end{itemize}

Now we show the expression $ \mathbb{E}[d(\beta, \gamma)] \le \mathbb{E}[d(\beta,\beta')]$ holds, as the following inequalities are equivalent:
\begin{equation}
    \begin{split}
        \mathbb{E}[d(\beta, \gamma)] \le & \mathbb{E}[d(\beta,\beta')]  \\
\sum_{i<j} b_{ij} + g_{ij} - 2 b_{ij}\cdot g_{ij}  \le & \sum_{i < j} 2b_{ij} - 2b_{ij}^2 \\
\sum_{i<j} g_{ij} - 2  b_{ij}\cdot g_{ij} \le & \sum_{i < j} b_{ij} - 2b_{ij}^2 \\
\sum_{i < j} g_{ij} - 2 \cdot (2g_{ij}^2 + 2 \epsilon_{ij} \cdot g_{ij}) \le & \sum_{i < j} g_{ij} + \epsilon_{ij} -  2(g_{ij}^2 + 2 g_{ij} \cdot \epsilon_{ij} + \epsilon_{ij}^2) \\
\sum_{i<j} - 2 g_{ij}^2  \le & \sum_{i<j} \epsilon_{ij} - 2 \epsilon_{ij}^2 \\
\sum_{i<j} \epsilon_{ij} (\epsilon_{ij} - \frac{1}{2})  \le & \sum_{i<j} g_{ij}^2.
    \end{split}
\end{equation}

Which conclude the proof of the right hand side, as the last inequality is always true since $\forall i < j$, $g_{ij}^2 \ge 0$ and ${\epsilon_{ij}(\epsilon_{ij}-\frac{1}{2}) \le 0}$.

For the left hand side: Using once again Corollary 3 from \cite{DBLP:conf/icml/Busa-FeketeHS14} which states that $\forall i <j$, $1- b_{ij} < \frac{1}{2} < b_{ij}$, we have:
\begin{align*}
\mathbb{E}[d(\beta, \sigma_0)] & = \sum_{i < j} (1-b_{ij}) 
 = \sum_{i<j} g_{ij}(1-b_{ij}) + (1-g_{ij})(1-b_{ij}) \\
& < \sum_{i < j} g_{ij}(1-b_{ij}) + b_{ij}(1-g_{ij}) 
 = \mathbb{E}[d(\beta,\gamma)].
\end{align*}
\end{proof}

\section{Proof of Theorem~\ref{thm:2pops}}

\begin{proof}

Note that
\begin{equation}
\begin{split}
	 \mathbb{E}[\delta_{\beta}] = (1-r) \cdot \mathbb{E}[d(\beta, \beta')] + r \cdot \mathbb{E}[d(\beta, \gamma)] \\
	\mathbb{E}[\delta_{\gamma}] = r \cdot \mathbb{E}[d(\gamma, \gamma')] + (1-r) \cdot \mathbb{E}[d(\beta, \gamma)]
\end{split}
\end{equation}

And our goal is to show there is a lower bound for the difference \begin{equation}\label{eq:diff_deltas}
\begin{split}
    \mathbb{E}[& \delta_{\beta}- \delta_{\gamma}] = (1-r)  \mathbb{E}[d(\beta, \beta')] + (2r - 1)  \mathbb{E}[d(\gamma, \beta)] - r \mathbb{E}[d(\gamma, \gamma')]. 
\end{split}
\end{equation}

This proof is divided in two parts. First, we show that the following expression holds: 
\begin{equation}\label{eq:sep_first_ineq}
c \cdot r \cdot \mathbb{E}[d(\gamma,\sigma_0)] \le (1-r) \cdot \mathbb{E}[d(\beta, \beta')] + (2r - 1) \cdot \mathbb{E}[d(\gamma, \beta)].
\end{equation}

In order to show the correctness of the above expression, we will deal with cases where $r < 0.5$ and $r \ge 0.5$ separately.

Case $r < 0.5$:

Starting from the right hand side of Lemma 3 and multiplying by $(2r-1)$ (negative in this case) and adding $(1-r)\mathbb{E}[d(\beta,\beta')]$ on both sides it holds that:
\begin{equation}
\begin{split}
 (1-r)  & \mathbb{E}[d(\beta, \beta')] + (2r - 1)  \mathbb{E}[d(\gamma, \beta)] 
\ge  (1-r)  \mathbb{E}[d(\beta, \beta')] + (2r - 1) \mathbb{E} [d(\beta, \beta')] \\
= & r \cdot \mathbb{E}[d(\beta, \beta')] 
\ge r \cdot \mathbb{E}[d(\beta, \sigma_0)] 
\ge  c \cdot r \cdot \mathbb{E}[d(\gamma, \sigma_0)],
\end{split}
\end{equation}

where the last two inequalities are obtained from the right hand side of Equation (6) of the original paper and the assumption that $\mathbb{E}[d(\beta, \sigma_0)] \ge c \cdot \mathbb{E}[d(\gamma, \sigma_0)]$ respectively. 

Case $r \ge 0.5$:

Starting from the result of the left hand side of  Lemma 3 and multiplying by $(2r-1)$ (positive in this case) and adding $(1-r)\mathbb{E}[d(\beta,\beta')]$ on both sides it holds that:

\begin{equation}
\begin{split}
(1-r) & \mathbb{E}[d(\beta, \beta')] + (2r - 1)  \mathbb{E}[d(\gamma, \beta)]  \ge  (1-r) \mathbb{E}[d(\beta, \beta')] + (2r - 1) \mathbb{E}[d(\beta, \sigma_0)] \\
\ge & (1-r) \mathbb{E}[d(\beta, \sigma_0)] + (2r - 1) \mathbb{E}[d(\beta, \sigma_0)] 
=  r \cdot \mathbb{E}[d(\beta, \sigma_0)] \\
\ge & c \cdot r \cdot \mathbb{E}[d(\gamma, \sigma_0)],
\end{split}
\end{equation}

where the last two inequalities are obtained from the right hand side of Equation (6) of the original paper and the assumption that $\mathbb{E}[d(\beta, \sigma_0)] \ge c \cdot \mathbb{E}[d(\gamma, \sigma_0)]$ respectively.

This finishes the first part, in which we show that Equation \eqref{eq:sep_first_ineq} holds for any value of $r$. In the second part we will add it to the following result, obtained using the left hand side of Equation (6) of the paper, by multiplying it by $-2r$:

\begin{equation}
-2r \mathbb{E}[d(\gamma, \sigma_0)] \le -r \mathbb{E}[d(\gamma, \gamma')].
\end{equation}

Hence, we finally have, using Equation \eqref{eq:diff_deltas}: 
\begin{equation}
\begin{split}
( c  - 2) \cdot r \cdot \mathbb{E}[d(\gamma, \sigma_0)] 
\le  (1-r)  \mathbb{E}[d(\beta, \beta')] + (2r - 1)  \mathbb{E}[d(\gamma, \beta)] - r \mathbb{E}[d(\gamma, \gamma')] 
=  \mathbb{E}[\delta_{\beta}-\delta_{\gamma}].
\end{split}
\end{equation}

\end{proof}

\section{Proof of Theorem~\ref{thm:clustering}}

\begin{proof}
First, we show that for $d(\sigma, \sigma')$ $i.i.d.$ random variables with range $\frac{n(n-1)}{2}$ then, it holds that with probability at least $1 - \epsilon$: 

\begin{equation}
\label{eq:hoeffding}
| \delta_{\sigma} - \mathbb{E}[d(\sigma,\sigma')] | \le \frac{n(n-1)}{2}\sqrt{\frac{\log(2/\epsilon)}{2(m-1)}}.
\end{equation}
This is a direct application of Hoeffding's inequality  which states that for any $t$ real-valued $i.i.d$ random variables $X_1, \dots, X_t$ of range $R$ and mean $\mu$, it holds with probability $1-\epsilon$ that: \\

$$|\bar{X_t} - \mu| \le R \sqrt{\frac{\log(2/\epsilon)}{2t}}, $$

where $\bar{X_t} = \frac{1}{t} \sum_{i=1}^t X_i $. 

Now, we consider the Single-linkage clustering, which separates two components when
\begin{equation}
\label{eq:separation_condition}
    \mathbb{E}[\delta_{\beta} - \delta_{\gamma}] \ge |\delta_{\beta} - \mathbb{E}[\delta_{\beta}]| + |\delta_{\gamma} - \mathbb{E}[\delta_{\gamma}]|,
\end{equation}

therefore, the theorem is holding if Equation \eqref{eq:separation_condition} holds. Hence, using Equation~\eqref{eq:hoeffding}, if we have 
\begin{equation}
\begin{split}
|\delta_{\beta} - \mathbb{E}[\delta_{\beta}]| + |\delta_{\gamma} - \mathbb{E}[\delta_{\gamma}]| & \le 2 \cdot \frac{n(n-1)}{2} \sqrt{\frac{\log(2/\epsilon)}{2m}} 
 \le (c-2) \cdot r \cdot \mathbb{E}[\gamma, \sigma_0] \\
& \le \mathbb{E}[\delta_{\beta} - \delta_{\gamma}],
\end{split}
\end{equation}

where we suppose the second inequality in order to have Equation \eqref{eq:separation_condition} and where the last inequality is coming from Theorem 4.
We obtain a bound for $m$ using the second inequality, which conclude the proof.

\end{proof}

\section{Preliminaries for sample complexity proofs}

\begin{lem}
\label{lem:symmetry}
Let $p(\sigma)$ be the probability of ranking $\sigma\in S_n$ under the Mallows model. For every $1\leq i,k\leq n$ the following holds. 
\begin{equation}
\sum_{\sigma:\sigma(i) \leq k} p(\sigma) = 1 - \sum_{\sigma:\sigma(n-i) \leq n-k}p(\sigma).
\label{eq:sym}
\end{equation}
\end{lem}

\begin{proof}
We are going to do a constructive bijection between the permutations in the set $\{\sigma:\sigma(i) \leq k\}$ and those in  the set $\{\sigma:\sigma(n-i) \geq n-k\}$ and show that the probabilities are the same. 
Let us first define these sets. 
\begin{itemize}
\item $S = \{\sigma:\sigma(i) \leq k\}$ 
\item $S'=\{\sigma:\sigma(i) > n-k\}$
\item $S''=\{\sigma:\sigma(n-i) > n-k\}$
\end{itemize}

\textbf{Flip step.} First, note that there is bijection between the sets $S$ and  $S'$. 
Given a permutation $\sigma \in S$ we can construct a permutation $\sigma' \in S'$ by setting $\sigma' (i) = n-\sigma(i)+1$. Moreover, $d(\sigma') = {n\choose 2}-d(\sigma)$. 

\textbf{Reverse step.} Now, we show a bijection between $S'$ and $S''$. Given a permutation $\sigma' \in S'$ we can construct a permutation $\sigma'' \in S''$ by setting $\sigma'' (i) = \sigma'(n-i)$. Moreover, $d(\sigma'') = {n\choose 2}-d(\sigma') = d(\sigma)$. This implies that $\sum_{\sigma\in S}p(\sigma) = \sum_{\sigma''\in S''}p(\sigma'')$. 

\textbf{Comlementary step.} Since ${S'' \cup \{\sigma:\sigma(n-i) \leq n-k\} = S_n}$ and both sets are disjoint, to conclude the proof, we just have to note that
\begin{equation}
\sum_{\sigma''\in S''}p(\sigma'') = 1 - \sum_{\sigma:\sigma(n-i) \leq n-k}p(\sigma),
\end{equation}

 which, in turn, implies the results in Equation~\eqref{eq:sym}. 

\end{proof}

\begin{lem} \label{thm:delta_sym}
Let $\Delta^{ik}$ be defined as is Definition 6 where $p(\sigma)$ is the probability of ranking $\sigma\in S_n$ under the Mallows model. For every $1\leq i,k\leq n$ the following holds. 
\begin{equation}
\Delta^{ik} = \Delta^{n-i-1,n-k}
\end{equation}
\end{lem}
\begin{proof}
Based on the result on Lemma~\ref{lem:symmetry}, we can rewrite $\Delta$ in this way. 
\begin{equation}
\begin{split}
\Delta^{ik}  =&  \sum_{\sigma:\sigma(i) \leq k}p(\sigma) - \sum_{\sigma:\sigma(i+1) \leq k}p(\sigma)  =  
  \Big(1-\sum_{\sigma:\sigma(n-i) \leq n-k}p(\sigma) \Big )- \Big(1- \sum_{\sigma:\sigma(n-i-1) \leq n-k}p(\sigma) \Big)   \\
=&  \sum_{\sigma:\sigma(n-i-1)<n-k}p(\sigma)  - \sum_{\sigma:\sigma(n-i)<n-k}p(\sigma)   = \Delta^{n-i-1,n-k}.
\end{split}
\end{equation}
\end{proof}

\section{Proof of Lemma~\ref{thm:min_Delta}}
\begin{proof}
The symmetry described in Lemma~\ref{thm:delta_sym} implies a symmetry in $\arg_{i}\min\Delta^{ik}$ that allows us focusing on the case $k\geq n/2$. For $k\geq n/2$ 
$\arg_{i}\min\Delta^{ik} = 1$. As a summary, 

\begin{equation}
\arg\min_{i}\Delta^{ik} =
\begin{cases}
n-1 \quad &\text{if } k<n/2, \\
1 \quad &\text{otherwise}
\end{cases}
\end{equation}
which, in turn implies

\begin{equation}
\min_i\Delta^{ik} =
\begin{cases}
\Delta^{1,n-k}  &\text{if } k<n/2, \\
\Delta^{1,k}  &\text{otherwise}.
\end{cases}
\end{equation}

Despite there is not a closed form expression for $\Delta^{nk}$, $\Delta^{1k} $ can be computed in $O(k^2)$. 
\begin{equation}
\begin{split}
& \Delta^{1k} =  \sum_{\sigma:\sigma(1) \leq k}p(\sigma) - \sum_{\sigma:\sigma(2) \leq k}p(\sigma)   \\
&=\sum_{r_1=0}^{k-1}p(V_1=r_1)  - 
\Big ( \sum_{r_1=0}^{k-1} \sum_{r_2=0}^{k-2} p(V_1=r_1) p(V_2=r_2)  +   \sum_{r_1=k}^{n} \sum_{r_2=0}^{k-1} p(V_1=r_1) p(V_2=r_2)\Big)\\
&=  \sum_{r_1=0}^{k-1}  \frac{\exp(-\theta r_1) }{\psi_{n,1}} 
 - \sum_{r_1=0}^{k-1} \sum_{r_2=0}^{k-2} \frac{\exp(-\theta r_1) \exp(-\theta r_1)}{\psi_{n,1}\psi_{n,2}}  - \sum_{r_1=k}^{n} \sum_{r_2=0}^{k-1}  \frac{\exp(-\theta r_1) \exp(-\theta r_2)}{\psi_{n,1}\psi_{n,2}},
    \end{split}
\end{equation}
which concludes the proof. 
\end{proof}

\section{Proof of Theorem~\ref{thm:sample_complexity_expertness}}

\begin{proof}

Borda outputs the correct order for the pair of items $i$ and $i+1$ with probability $1-\epsilon$ when the number of permutations is at least 
\begin{equation}
m \geq \frac{2 n^2 \log \epsilon^{-1}}{(\sum_{j=1}^{n-1}  \Delta^{ij} )^2}.
\label{eq:sample_complex}
\end{equation}

This expression has been used in sample complexity results that do not consider the spread parameter~\cite{Caragiannis2013}. The authors use this expression for $\Delta^{ij}$ which can be shown to be equivalent. \footnote{Note that the definitions of $m$ and $n$ are interchanged in their paper and that they denote the dispersion in the model as $\phi = \exp(-\theta)$. } 

\begin{equation}
\begin{split}
\sum_{j=1}^{n-1} ( \sum_{l=1}^{j} \sum_{\sigma:\sigma(i) = l } p(\sigma)  -  \sum_{l=1}^{j} \sum_{\sigma:\sigma(i+1) = l } p(\sigma)  ) 
= \sum_{j=1}^{n-1} ( \sum_{\sigma:\sigma(i) \leq j } p(\sigma)  -   \sum_{\sigma:\sigma(i+1) \leq j } p(\sigma)  )  =   \sum_{j=1}^{n-1}  \Delta^{ij}.
\end{split}
\end{equation}



In these lines we extend the result by bounding the number of samples (1) w.r.t. the dispersion parameter and (2) considering top-$k$ rankings. 
To prove these points, we give an upper bound for $\sum_{j=1}^{n-1}  \Delta^{ij}$, which is a function of the dispersion parameter $\theta$.

Assume w.l.o.g. that $\sigma_0 = e$, let $\tau_i$ be an inversion of positions $i$ and $i+1$, i.e., $\tau_i(i)=i+1$, $\tau_i(i+1)=i$ and $\tau_i(j)=j$ for $j\neq i,i+1$. As for any inversion, the result of the composition $\sigma\tau_i$ swaps positions $i$ and $i+1$ in $\sigma$. Therefore, 

\begin{equation}
\begin{split}
\sum_{j=1}^{n-1} & \Delta^{ij} 
= \sum_{\{\sigma:\sigma(i) < \sigma(i+1)\}}( p(\sigma) -p(\sigma\tau_i) ( \sigma(i+1) - \sigma(i) )
 \leq \sum_{\sigma:\sigma(i)<\sigma(i+1) } ( p(\sigma) -p(\sigma\tau_i) ) k \\
& = k \Big ( \sum_{\substack{\sigma:\sigma(i) \leq k \,\land \\ \sigma(i)<\sigma(i+1) }}  p(\sigma) -p(\sigma\tau_i)  
	       + \sum_{\substack{\sigma:\sigma(i+1)>k \,\land \\ \sigma(i)<\sigma(i+1)  }}  p(\sigma) -p(\sigma\tau_i)  \Big)\\ 
\end{split}
\end{equation}

Note that for partial permutations of $k$ items, the second sum equals 0. Therefore, the following is equivalent
\begin{equation}
\begin{split}
& = k  \sum_{\substack{\sigma:\sigma(i) \leq k \,\land \\ \sigma(i)<\sigma(i+1) }} p(\sigma) -p(\sigma\tau_i)   = k  \sum_{\substack{\sigma:\sigma(i) \leq k \,\land \\ \sigma(i)<\sigma(i+1) }}p(\sigma) (1-\exp(-\theta))  \\ %
& \leq k (1-\exp(-\theta)) \sum_{\sigma:\sigma(i) \leq k } p(\sigma).   \\ %
\end{split}
\end{equation}
Since $\sum_{\sigma:\sigma(i) \leq k } p(\sigma) $ decreases w.r.t. $i$ the following is equivalent. 


\begin{equation}
\begin{split}
k &(1-\exp(-\theta)) \sum_{\sigma:\sigma(i) \leq k } p(\sigma)
 \leq k (1-\exp(-\theta)) \sum_{\sigma:\sigma(1) \leq k } p(\sigma) - i \min \Delta^{ik}  \\ %
& \leq k (1-\exp(-\theta)) \sum_{r \leq k} p(V_1=r)  - i \Delta^{1k} \leq k^2 (1-\exp(-\theta)) p(V_1=0)   - i \Delta^{1k}  \\ %
& =   \frac{ k^2 (1-\exp(-\theta))^2}{1-\exp(-\theta n)} - i \Delta^{1k}.  \\ %
\end{split}
\end{equation}

It follows that 
\begin{equation}
\sum_{j=1}^{n-1}\Delta^{ij} \leq \frac{ k^2 (1-\exp(-\theta))^2}{1-\exp(-\theta n)} - i \Delta^{1k} ,
\end{equation}

and therefore, Borda outputs the true ranking $\sigma_0$ with probability $1-\epsilon$ when the number of samples is at least
\begin{equation}
m \geq 2 n^2 \log \epsilon^{-1} \Big (\frac{ k^2 (1-\exp(-\theta))^2}{1-\exp(-\theta n)} - i \Delta^{1k} \Big ) ^{-2},
\end{equation}
which concludes the proof.

\end{proof}

\bibliographystyle{plain}
\bibliography{mendeley}